\PassOptionsToPackage{unicode}{hyperref}
\PassOptionsToPackage{hyphens}{url}
\documentclass[
  1p]{elsarticle/elsarticle}
\usepackage{lmodern}
\usepackage{amssymb,amsmath}
\usepackage{ifxetex,ifluatex}
\ifnum 0\ifxetex 1\fi\ifluatex 1\fi=0 
  \usepackage[T1]{fontenc}
  \usepackage[utf8]{inputenc}
  \usepackage{textcomp} 
\else 
  \usepackage{unicode-math}
  \defaultfontfeatures{Scale=MatchLowercase}
  \defaultfontfeatures[\rmfamily]{Ligatures=TeX,Scale=1}
\fi
\IfFileExists{upquote.sty}{\usepackage{upquote}}{}
\IfFileExists{microtype.sty}{
  \usepackage[]{microtype}
  \UseMicrotypeSet[protrusion]{basicmath} 
}{}
\makeatletter
\@ifundefined{KOMAClassName}{
  \IfFileExists{parskip.sty}{%
    \usepackage{parskip}
  }{
    \setlength{\parindent}{0pt}
    \setlength{\parskip}{6pt plus 2pt minus 1pt}}
}{
  \KOMAoptions{parskip=half}}
\makeatother
\usepackage{xcolor}
\IfFileExists{xurl.sty}{\usepackage{xurl}}{} 
\IfFileExists{bookmark.sty}{\usepackage{bookmark}}{\usepackage{hyperref}}
\hypersetup{
  pdftitle={Next Priority Concept: A new and generic algorithm computing concepts from complex and heterogeneous data},
  pdfkeywords={Formal Concept Analysis, Lattice, Pattern Structure, Strategies, Heterogeneous data},
  hidelinks,
  pdfcreator={LaTeX via pandoc}}
\usepackage{longtable,booktabs}
\usepackage{etoolbox}
\makeatletter
\patchcmd\longtable{\par}{\if@noskipsec\mbox{}\fi\par}{}{}
\makeatother
\IfFileExists{footnotehyper.sty}{\usepackage{footnotehyper}}{\usepackage{footnote}}
\makesavenoteenv{longtable}
\usepackage{graphicx,grffile}
\makeatletter
\def\maxwidth{\ifdim\Gin@nat@width>\linewidth\linewidth\else\Gin@nat@width\fi}
\def\maxheight{\ifdim\Gin@nat@height>\textheight\textheight\else\Gin@nat@height\fi}
\makeatother
\setkeys{Gin}{width=\maxwidth,height=\maxheight,keepaspectratio}
\makeatletter
\def\fps@figure{htbp}
\makeatother
\providecommand{\tightlist}{%
  \setlength{\itemsep}{0pt}\setlength{\parskip}{0pt}}
\usepackage{algorithm2e}
\usepackage{galois}
\newtheorem{theorem}{Theorem} \newtheorem{lemma}{Lemma} \newproof{proof}{Proof} \newproof{corollary}{Corollary} \newproof{property}{Property}
\usepackage[]{natbib}
\date{}

\begin{document}
\begin{frontmatter}

\title{Next Priority Concept: A new and generic algorithm computing concepts from complex and heterogeneous data}

\author[Group1]{Christophe Demko\corref{cor1}}

\ead{christophe.demko@univ-lr.fr}

\author[Group1]{Karell Bertet}

\ead{karell.bertet@univ-lr.fr}

\author[Group1]{Cyril Faucher}

\ead{cyril.faucher@univ-lr.fr}

\author[Group1]{Jean-François Viaud}

\ead{jean-francois.viaud@univ-lr.fr}

\author[Group2]{Sergeï Kuznetsov}

\ead{skuznetsov@yandex.ru}

\address[Group1]{L3i, La Rochelle University, France}
\address[Group2]{Faculty of Computer Sciences, Moscow, Russia}

\cortext[cor1]{Corresponding author}

\begin{abstract}
In this article, we present a new data type agnostic algorithm calculating a concept lattice from heterogeneous and complex data. Our \textsc{NextPriorityConcept} algorithm is first introduced and proved in the binary case as an extension of Bordat's algorithm with the notion of strategies to select only some predecessors of each concept, avoiding the generation of unreasonably large lattices. The algorithm is then extended to any type of data in a generic way. It is inspired from pattern structure theory, where data are locally described by predicates independent of their types, allowing the management of heterogeneous data.
\end{abstract}


\end{frontmatter}

\hypertarget{introduction}{%
\section{Introduction}\label{introduction}}

Formal Concept Analysis (FCA) is a branch of applied lattice theory,
which originated from the study of relationship between Galois connections,
closure operators, and orders of closed sets \citep{Monjardet_1970, Ganter_1999}.

Starting from a binary relation between a set of objects and a set of
attributes, formal concepts are built as maximal sets of objects in relation
with maximal sets of attributes, by means of derivation operators forming a
Galois connection whose composition is a closure operator \citep{Bertet_2018}.
Concepts form a partially ordered set that represents the initial data called
the concept lattice. This lattice has proved to be useful in many fields, e.g.
artificial intelligence, knowledge management, data-mining, machine learning,
etc.

Many extensions from the original formalism, which was based on binary data,
have been studied in order to work with non-binary data, such as numbers,
intervals, sequences, trees, and graphs. The formalism of pattern
structures \citep{Kuznetsov_2001, Kaytoue_2015, HdRMehdi} extends FCA to deal with non binary data
provided by space description organised as a semi-lattice
in order to maintain a Galois connection between objects and their descriptions.
Therefore a pattern lattice represents the data where concepts are composed of
objects together with their shared descriptions.

This space description must be organised and defined as a semi-lattice in
a preliminary step, independently of the data, often with a large number
of generated concepts and unreasonably large lattices that are uneasy to interpret.
Otherwise, pattern structures do not allow an easy management of heterogeneous datasets
where several kinds of characteristics describe data.

In this paper, we present the \textsc{NextPriorityConcept} algorithm
that computes a concept lattice from heterogeneous data, where:

\begin{itemize}
\tightlist
\item
  Patterns are locally selected and discovered:\\
  Indeed, patterns of each concept are locally discovered, and predecessors of a concept can be
  filtered according to a specific strategy. So patterns computed by our algorithm
  are more adapted to the data, and lattices are smaller.
\item
  Pattern mining for heterogeneous and complex data:
  These patterns are formalized by predicates whatever the description of data,
  then we can merge patterns issued from distinct space descriptions, and manage
  heterogeneous data in a generic and agnostic way.
\end{itemize}

\hypertarget{preliminaries}{%
\section{Preliminaries}\label{preliminaries}}

\hypertarget{formal-concept-analysis}{%
\subsection{Formal Concept Analysis}\label{formal-concept-analysis}}

Let \(\langle G,M,I \rangle\) a \emph{formal context} where \(G\) is a non-empty set
of objects, \(M\) is a non-empty set of attributes and \(I \subseteq G\times M\)
is a binary relation between the set of objects and the set of attributes. Let
\((2^G,{\subseteq})\galois{\alpha}{\beta}(2^M,{\subseteq})\) be the
corresponding \emph{Galois connection} where:

\begin{itemize}
\tightlist
\item
  \(\alpha: 2^G\rightarrow 2^M\) is an application which associates a subset
  \(B\subseteq M\) to every subset \(A\subseteq G\) such that
  \(\alpha(A)=\{b\;:\;b\in M \wedge \forall a\in A,aIb\}\);
\item
  \(\beta: 2^M\rightarrow 2^G\) is an application which associates a subset
  \(A\subseteq G\) to every subset \(B\subseteq M\) such that
  \(\beta(B)=\{a\;:\;a\in G \wedge \forall b\in B,aIb\}\).
\end{itemize}

A concept is a pair \((A,B)\) such that \(A\subseteq G\), \(B\subseteq M\),
\(B=\alpha(A)\) and \(A=\beta(B)\). The set \(A\) is called the \emph{extent}, whereas \(B\)
is called the \emph{intent} of the concept \((A,B)\). There is a natural hierarchical
ordering relation between the concepts of a given context that is called
the subconcept-superconcept relation:

\[(A_1,B_1)\leq (A_2,B_2) \iff A_1 \subseteq A_2 (\iff B_2 \subseteq B_1)\]

The ordered set of all concepts makes a complete lattice called the \emph{concept
lattice} of the context, that is, every subset of concepts has an infimum (meet)
and a supremum (join).

\hypertarget{a-basic-algorithm}{%
\subsection{A basic algorithm}\label{a-basic-algorithm}}

Bordat's theorem \citep{Bordat_1986} states that there is a bijection between the immediate
successors of a concept \((A,B)\) and the inclusion maximal subsets of the
following family defined on the objects \(G\):

\begin{equation}
\label{eq:FAB} {\cal FS}_{(A,B)}=\{\alpha(a)\cap B\;:\; a\in G \setminus A\}
\end{equation}

Bordat's algorithm \citep{Bordat_1986}, that we also find in Linding's work \citep{Linding_2002},
is a direct implementation of Bordat's theorem.
This algorithm recursively computes the Hasse diagram of the concept lattice of a context
\(\langle G,M,(\alpha,\beta)\rangle\) starting from the bottom concept
\((\beta(M),M)\) and by computing at each recursive call the immediate successors of a
concept \((A,B)\): the family \({\cal FS}_{(A,B)}\) is first computed, then the
inclusion maximal sets are selected.

\({\cal FS}_{(A,B)}\) is composed of the intent part of the immediate potential successors of \((A,B)\).
Each intent \(B'\) is obtained by \(\alpha(a)\cap B\), where \(a\in G\setminus A\)
is a new potential object for a successor concept of \((A,B)\).
Indeed, as defined by the order relation between concepts, immediate successors of \((A,B)\)
are obtained by an increase of \(A\) by at least one new potential object \(a\), and thus a reduction of \(B\)
to \(B'=\alpha(a)\cap B\), and clearly \(B'\subset B\).
Moreover, \(B'\) must be maximal by inclusion in \({\cal FS}_{(A,B)}\), otherwise there exists \(B''\in {\cal FS}_{(A,B)}\)
such that \(B'\subset B'' \subset B\), then \(B'\) is not the intent of an immediate successor of \((A,B)\).
If \(B'\) is inclusion maximal in \({\cal FS}_{(A,B)}\),
then \((\beta(B'),B')\) is an immediate successor of \((A,B)\).

Our next priority algorithm focuses on the objects and the dual version of Bordat's theorem.
It considers the whole set \(G\) of objects at the beginning.
Then, for each concept \((A,B)\), it does not test a new potential object,
but a new potential attribute \(b\in M\setminus B\) describing a subset of \(A\).
In this way, \(A\) decreases while \(B\) increases, that corresponds to the predecessor relation.
This process corresponds to the dual version of Bordat's theorem stating that the
immediate predecessors of \((A,B)\) are the maximal inclusion subsets of
the following family on the attributes \(M\):

\begin{equation}
\label{eq:FAB} {\cal FP}_{(A,B)}=\{\beta(b)\cap A\;:\; b\in M \setminus B\}
\end{equation}

The \textsc{NextPriorityConcept‑Basic} algorithm is a new version of Bordat's algorithm
where recursion is replaced by a priority queue using the support of concepts.
And, at each iteration, the concept \((A,B)\) of maximal support is produced, then its immediate
predecessors are computed by \textsc{Predecessors‑Basic} (\((A,B)\)) that returns the
inclusion maximal sets of \({\cal FP}_{(A,B)}\), and then are stored in the priority queue.
Therefore concepts are generated level by level, starting from the top concept
\((G,\alpha(G))\), and each concept is generated before its predecessors.

\begin{algorithm}[H]
\label{next-priority-concept-basic}
\SetKw{KwNot}{not}
\SetKw{KwEmpty}{empty}
\SetKw{KwDelete}{delete}
\SetKw{KwProduce}{produce}
\KwData{
\begin{itemize}
\item $\langle G,M,(\alpha,\beta)\rangle$ be a formal context
\end{itemize}
}
\KwOut{
\begin{itemize}
\item concepts $(A,B)$ of the formal context
\end{itemize}
}
\Begin{
    \texttt{/* Priority queue for the concepts */}\\
    $Q \leftarrow \mbox{\texttt{[]}}$ \tcc*{$Q$ is a priority queue using the support of concepts}
    $Q$.push($(|G|,G)$) \tcc*{Add the top concept into the priority queue}
    \texttt{}\\

    \While{$Q$ \KwNot \KwEmpty}{
        \texttt{/* Compute concept */}\\
        $A \leftarrow Q.\mbox{pop()}$ \tcc*{Get the concept with highest support}
        $B \leftarrow \alpha(A)$ \tcc*{Compute the intent of this concept}
        \KwProduce $(A,B)$ \;

        \texttt{}\\
        \texttt{/* Update queue */}\\
        $L \leftarrow \mbox{\textsc{Predecessors‑Basic}}((A,B), M, (\alpha, \beta))$ \;
        \ForAll{$A' \in L$}{
            $Q$.push($(|A'|,A')$) \tcc*{Add concept into the priority queue}
        }
    }
}
\caption{\textsc{NextPriorityConcept‑Basic}}
\end{algorithm}

\begin{algorithm}[H]
\label{predecessors-basic}
\SetKw{KwTrue}{true}
\SetKw{KwFalse}{false}
\SetKw{KwBreak}{break}
\KwData{
\begin{itemize}
\item $(A,B)$ a concept
\item $M$ the set of attributes
\item $(\alpha,\beta)$ the Galois connection
\end{itemize}
}
\KwResult{
\begin{itemize}
\item $L$ a set of predecessors of $(A,B)$ represented by their extent
\end{itemize}
}
\Begin{
    $L \leftarrow \emptyset$\;

    \ForAll{$b \in M \setminus B$}{
        $A' \leftarrow \beta(b) \cap A$ \tcc*{Extent of a new potential predecessor}
        $L\leftarrow$ \textsc{Inclusion‑Max}($L$,$A'$) \tcc*{Add $A'$ if maximal in $L$}
    }
    \Return{$L$}
}
\caption{\textsc{Predecessor‑Basic}}
\end{algorithm}

\newpage

\begin{algorithm}[H]
\label{inclusion-max}
\SetKw{KwTrue}{true}
\SetKw{KwFalse}{false}
\SetKw{KwBreak}{break}
\KwData{
\begin{itemize}
\item $L$ a set of potential predecessors represented by their extent
\item $A$ a new potential predecessor
\end{itemize}
}
\KwResult{
\begin{itemize}
\item inclusion maximal subsets of $L+A$
\end{itemize}
}
\Begin{
        $add \leftarrow $ \KwTrue \;
        \ForAll{$A' \in L$}{
            \uIf{$A \subset A'$}{
                $add \leftarrow $ \KwFalse \tcc*{$A$ is not an immediate predecessor}
                \KwBreak
            }
            \ElseIf{$A' \subset A$}{
                $L$.remove($A'$) \tcc*{Remove $A'$ as a possible predecessor}
            }
        }
        \lIf(\tcc*[f]{Add $A$ as a predecessor}){add}{$L$.add($A$)}
        \Return{$L$}
}
\caption{\textsc{Inclusion‑Max}}
\end{algorithm}

This non-recursive version of Bordat's algorithm preserves its complexity.
Therefore we can state the following result:

\begin{property}

Algorithm \textsc{NextPriorityConcept‑Basic} computes the concept lattice of
\(\langle G,M,(\alpha,\beta)\rangle\) in
\(O(|G||M|^2)\).

\end{property}

We will introduce our \textsc{NextPriorityConcept} algorithm in two steps:

In Section \ref{nextpriorityalgorithm-strategy},
\textsc{NextPriorityConcept‑Basic} algorithm is first modified in order to introduce the possibility
to filter the new attributes considered during the immediate predecessor process according to a \emph{strategy} \(\sigma\) of exploration.

In Section \ref{nextpriorityalgorithm-heterogeneous},
the final version of \textsc{NextPriorityConcept}, inspired from pattern structures,
extends the computation of concepts to heterogeneous
dataset, where attributes \(P\) are predicates deduced from each
characteristics of data according to a specific \emph{description}.

\hypertarget{nextpriorityalgorithm-strategy}{%
\section{Next Priority Concept: filtering of concepts according to a strategy}\label{nextpriorityalgorithm-strategy}}

\hypertarget{algorithm-strategy}{%
\subsection{Extension of the algorithm with strategies}\label{algorithm-strategy}}

Lattices are often unreasonably too large, which hinders their ability to
provide readability and explanation of the data. In this section, we extend
the basic \textsc{NextPriorityConcept‑Basic} algorithm to select only some predecessors at each
iteration.
Rather than considering all the attributes of \(M \setminus B\) to calculate
the potential predecessor of a concept \((A,B)\), we apply a filter on these candidate attributes.
For example, we can select attributes of maximal support,
or according to class information as explained later.

More formally, a \emph{strategy} is an input application
\(\sigma:2^G\rightarrow 2^M\) which associates a subset
\(S\subseteq M\) of selected attributes to every subset \(A\subseteq G\).
Many strategies are possible. Let us introduce as examples the maximal support
strategy \(\sigma_{\mbox{\small max}}\) and the entropy strategy
\(\sigma_{\mbox{\small entropy}}\):

\begin{itemize}
\item
  The maximal support strategy relies on the support of attributes:

  \[\sigma_{\mbox{\small max}}(A)=\{b\in M \setminus \alpha(A)\;:\;
    |\beta(\alpha(A)+b)|\mbox{ maximal}\}\]
\item
  The entropy strategy is a supervised strategy where objects have a \emph{class}
  attribute:

  \[\sigma_{\mbox{\small entropy}}(A)=\{b\in M \setminus \alpha(A)\;:\;
    H_{\mbox{class}}(\beta(\alpha(A)+b))\mbox{ minimal}\}\]
\end{itemize}

We introduce a new set \(P\) of selected attributes according to the strategy,
and to avoid confusion, we will denote \((A,D)\) a concept defined on \(G\times P\),
and \(I_P\) the corresponding relation between \(G\) and \(P\).

To ensure that meets are correctly generated, we introduce a constraints
propagation mechanism \({\cal C}\) that associates a set
of attributes \({\cal C}[A]\) to process with each concept \((A,D)\).

\begin{algorithm}[h]
\label{next-priority-concept}
\SetKw{KwNot}{not}
\SetKw{KwEmpty}{empty}
\SetKw{KwDelete}{delete}
\SetKw{KwProduce}{produce}
\KwData{
\begin{itemize}
\item $\langle G,M,(\alpha,\beta)\rangle$ a formal context
\item $\sigma$ a strategy
\end{itemize}
}
\KwOut{
\begin{itemize}
\item the formal context $\langle G,P,I_P\rangle$
\item its concepts $(A,D)$
\end{itemize}
}
\Begin{
    \texttt{/* Priority queue for the concepts */}\\
    $Q \leftarrow \mbox{\texttt{[]}}$ \tcc*{$Q$ is a priority queue using the support of concepts}
    $Q$.push($(|G|,(G,\alpha(G)))$) \tcc*{Add the top concept into $Q$}

    \texttt{}\\
    \texttt{/* Data structure for constraints */}\\
    ${\cal C} \leftarrow \mbox{\texttt{[]}} $ \tcc*{${\cal C}$ is the descendant constraints map being $\emptyset$ by default}

    \texttt{}\\
    \texttt{/* Data structures for new attributes */}\\
    $P \leftarrow \emptyset$ \tcc*{$P$ is the set of selected attributes}
    $I_P \leftarrow \emptyset$ \tcc*{$I_P$ is the new binary relation between $G$ and $P$}

    \texttt{}\\
    \texttt{/* Immediate predecessors generation */}\\     
    \While{$Q$ \KwNot \KwEmpty}{
        \texttt{/* Compute concept */}\\
        $(A,D) \leftarrow Q.\mbox{pop()}$ \tcc*{Get the concept with highest support}
        \KwProduce $(A,D)$ \;

        \texttt{}\\
        $LP \leftarrow$ \textsc{Predecessors‑Strategy}($(A,D)$, $P$, $I_P$, ${\cal C}$, $\sigma$) \;

        \texttt{}\\
        \texttt{/* Update queue */}\\
        \ForAll{$(A',D') \in LP$}{
            \If {$(A',D')\not\in Q$}{
            $Q$.push($(|A'|,(A',D'))$) \tcc*{Add new concept into $Q$}
            }
        }

        \texttt{}\\
        \KwDelete ${\cal C}[A]$ \tcc*{Remove useless data}
    }
    \Return {$\langle G,P,I_P\rangle$}
}
\caption{\textsc{NextPriorityConcept‑Strategy}}
\end{algorithm}

The \textsc{NextPriorityConcept‑Strategy} algorithm considers a formal
context \(\langle G,M,(\alpha,\beta)\rangle\) and a strategy \(\sigma\) as input,
and computes the concept lattice of \(\langle G,P,I_P\rangle\)
according to the input strategy in the same way of the
\textsc{NextPriorityConcept‑Basic} algorithm, with in addition the management
of the constraint propagation.

The \textsc{Predecessors‑Strategy} algorithm is a modified version of
\textsc{Predecessors‑Basic} to compute predecessors of a concept \((A,D)\), where we
only consider the attributes \(\sigma(A)\) given by the strategy, and the
attributes \({\cal C}[A]\) given by the constraint propagation mechanism, instead
of the whole set \(M \setminus D\) of attributes.

\textsc{Inclusion‑Max} is similar, with a minimal test on the subset of objects \(A'\),
but the list \(L\) is composed of pairs \((A',d)\) instead of subsets \(A'\).

\begin{algorithm}[h]
\label{predecessors-strategy}
\SetKw{KwTrue}{true}
\SetKw{KwFalse}{false}
\SetKw{KwBreak}{break}
\KwData{
\begin{itemize}
\item $(A,D)$ a concept
\item $P$ the set of selected attributes
\item $I_P$ the new relation
\item ${\cal C}$ the constraints
\item $\sigma$ a strategy
\end{itemize}
}
\KwResult{
\begin{itemize}
\item $LP$ a set of predecessors
\end{itemize}
}
\Begin{
    $L \leftarrow \emptyset$\; 
    \ForAll{$b \in (\sigma(A) \cup {\cal C}[A]) \setminus D$}{
        \texttt{/* b is a new "potential" attribute for a predecessor */}\\
        $A' \leftarrow \beta(b) \cap A$ \tcc*{Compute the objects of $D+ b$}
        \texttt{/* Add $(A',b)$ if $A'$ maximum in $L$ and included in $A$ */}\\
        \lIf{$A'\subset A$}{$L \leftarrow$ \textsc{Inclusion‑Max}($L$,$(A',b)$)}
        }
    $N\leftarrow \{b\;:\; (A',b)\in L\}$ \tcc*{$N$ is the set of new constraints}
        
    $LP \leftarrow \emptyset$\;

    \ForAll{$(A',b')\in L$}{
        \texttt{/* Update the selected attributes $P$ and the relation $I_P$*/}\\
        \If{$b'\in \sigma(A)$} {
        $P \leftarrow P \cup \{b'\}$ \tcc*{Update the set of selected attributes}
        }
        $D' \leftarrow \alpha(A')\cap P$ \tcc*{Compute the extent of $A'$}
        $LP$.add($(A',D')$) \tcc*{$(A',D')$ is a new concept}
        $I_P \leftarrow I_P \cup (A'\times D')$ \tcc*{Update the new relation}

        \texttt{/* Compute residual and cross constraints */}\\
        ${\cal C}[A'] \leftarrow {\cal C}[A']\cup {\cal C}[A] \cup N \setminus D'$\\
        }

\Return{$LP$}
}
\caption{\textsc{Predecessors‑Strategy}}
\end{algorithm}

\hypertarget{proof-complexity}{%
\subsection{Proof and complexity analysis}\label{proof-complexity}}

\hypertarget{proof-of-the-algorithm}{%
\subsubsection{Proof of the algorithm}\label{proof-of-the-algorithm}}

\hypertarget{th:algo}{}
\begin{theorem} \label{th:algo}

The \textsc{NextPriorityConcept‑Strategy} algorithm computes all the concepts
of \(\langle G,P,(\alpha_P,\beta_P)\rangle\), with a strategy \(\sigma\) as input.

\end{theorem}

Consider \((A_i,D_i)\) the concept generated at each iteration \(i\) of the main
loop. To prove the theorem, we have to prove the two following lemmas:

\begin{lemma}

Let \((A_i,D_i)\) be a concept generated at iteration \(i\), then \((A_i,D_i)\) is a
concept of \(\langle G,P,(\alpha_P,\beta_P)\rangle\)

\end{lemma}

\begin{proof}

The priority queue \(Q\) is initialized with \((|G|,(G,\alpha \circ \beta(\emptyset)))\), and
\((G,\alpha \circ \beta(\emptyset))=(A_0,D_0)\) corresponds to the top concept
on \(P\), i.e.~the concept of greatest support.

Let us introduce \(P_i\) the set of selected attributes \(P\) at each iteration \(i\).
Since \(P\) is updated with new selected attributes at each iteration \(i\),
we have \(P_0\subseteq P_1 \subseteq \ldots\subseteq P_i \ldots\subseteq P\)
with \(P_0\) being initialized with \(D_0=\alpha \circ \beta(\emptyset)\).

Let \((A_i,D_i)\) be the concept generated at iteration \(i\).
Let us prove that \((A_i,D_i)\) is a concept of the generated context \(\langle G,P,I_P \rangle\).

Clearly \(D_i\subseteq P_i\) and, since \((A_i\times D_i)\) is added in \(I_P\)
at iteration \(i\), then \((A_i,D_i)\) is a concept of the context \(\langle G,P_i,I_P \rangle\).
Therefore we have to prove that \((A_i,D_i)\) is also a concept of \(\langle G,P,I_P \rangle\).

If \((A_i,D_i)\) is not a concept of \(\langle G,P,I_P \rangle\), then there exists a selector \(p\in P\) such that
\(p\in \alpha(A_i)\) and \(p\not\in D_i\), thus \(p\not\in P_i\).
From \(p\in \alpha(A_i)\) we have \(A_i\subseteq \beta(p)\).
Let the iteration \(j\) that adds \(p\) in \(P\), and let \((A_j,D_j)\) the concept generated at iteration \(j\).
Then \(p\in D_j\subseteq P_j\) and \((A_j\times D_j)\) is added in \(I_P\), thus \(A_j=\beta(p)\).
From \(p\not\in P_i\) and \(p\in P_j\), we deduce \(j>i\).
From \(A_i\subseteq \beta(p)\) and \(A_j=\beta(p)\), we deduce \(A_i\subset A_j\)
and \(j<i\) since concept are generated according to the priority queue using
the support.
Thus a contradiction and \((A_i,D_i)\) is a concept of \(\langle G,P,I_P \rangle\).

\end{proof}

\begin{lemma}

Let \((A,D)\) be a concept of \(\langle G,P,(\alpha_P,\beta_P)\rangle\),
then there exists an iteration \(i\) such that \((A,D)=(A_i,D_i)\).

\end{lemma}

\begin{proof}

Let \((A,D)\) be a concept of \(\langle G,P,I_P \rangle\).
Then \((A,D)\) is the meet of \(\{(\beta(p), \alpha\circ\beta(p))_{p\in D}\}\).
Let us prove that this meet is generated by the algorithm.

In the case where \(|D|=0\), then \((A,D)\) is the top concept generated at the begining of the algorithm.
In the case where \(|D|=1\), then \((A,D)=(\beta(p),\alpha\circ\beta(p))\) generated
by the iteration \(i\) that adds \(p\) in \(P\).

In the case where \(|D|>1\), let \(p\neq p'\in D\) such that \(i<j\),
where \(i\) is the iteration that adds \(p\) in \(P\),
and \(j\) is the iteration that adds \(p'\) in \(P\).

Let \((A_i,D_i)\) be the concept generated at iteration \(i\) and
\((A_j,D_j)\) be the concept generated at iteration \(j\).
Then \((A_i,D_i)= (\beta(p),\alpha\circ\beta(p))\) and
\((A_j,D_j)= (\beta(p'),\alpha\circ\beta(p'))\).

Let us prove that the meet \((A_i,D_i)\wedge (A_j,D_j)\) is generated.
We have two cases:

\begin{itemize}
\item If $D_i\subset D_j$ then $(A_i,D_i)\leq (A_j,D_j)$ and $(A_j,D_j)$ 
is equal to $(A_i,D_i)\wedge (A_j,D_j)$. 
\item If $D_i\not\subset D_j$ then the iteration $i$ adds $p$ as constraint 
to all other concepts, thus $p$ belongs to the set of constraints of $A_j$, 
and is considered is the first loop of the \textsc{Predecessors‑Strategy} algorithm 
to generate a potential immediate predecessor 
$S=(\beta(D_j+p), \alpha\circ\beta(D_j+p))$ of $(A_j,D_j)$. 
We have two cases again. 
\begin{itemize}
\item In the case where $S$ is minimal per inclusion between all the potential 
immediate predecessors of $(A_j,D_j)$, then $S$ is generated as immediate predecessor 
and corresponds to the meet $(A_i,D_i)\wedge (A_j,D_j)$. 
\item In the case where $S$ is not minimal per inclusion, then $p$ belongs to the 
set of constraints of the generated immediate predecessors of $(A_j,D_j)$, 
and the meet $(A_i,D_i)\wedge (A_j,D_j)$ will be generated as a predecessor 
of an immediate predecessor of $(A_j,D_j)$. 
\end{itemize}
\end{itemize}

Therefore the meet \((A_i,D_i)\wedge (A_j,D_j)\) is generated thanks to the
constraints propagation mechanism.
This achieves the proof.

\end{proof}

Therefore we can state the following result considering the selectors \(p\in \sigma(A) \cup {\cal C}[A]\):

\begin{lemma}

There is a bijection between the immediate
predecessors of a concept \((A,D)\) and the inclusion maximal subsets of the
following family defined on the objects \(G\):

\begin{equation}
\label{eq:FAB} {\cal FD}_{(A,D)}=\{\{a\in A\;:\;p(a)\}:\; p\in(\sigma(A) \cup {\cal C}[A])\}
\end{equation}

\end{lemma}

\hypertarget{run-time-analysis}{%
\subsubsection{Run-time analysis}\label{run-time-analysis}}

Denote by \(\mathcal{B}\) the collection
of all formal concepts of \(\langle G,M,(\alpha,\beta)\rangle\) generated by the
strategy \(\sigma\), i.e.~the concepts of
\(\langle G,P,(\alpha_P,\beta_P)\rangle\); and by \(c_\sigma\) the cost of
the strategy for a concept.

\begin{itemize}
\tightlist
\item
  the \textsc{Predecessors‑Strategy} algorithm computes the predecessors of a
  concept. Its run-time complexity is in
  \(O(|G|\,|P|^2\,c_{\sigma})\) where
  \(O(|G|\,|P|^2)\) is the cost of Bordat's algorithm \citep{Bordat_1986}.
  And the descendant constraints are updated in \(O(|P|^2)\).
\item
  the \textsc{NextPriorityConcept‑Strategy} algorithm updates the priority
  queue in \(O(|G|\,|P|)\).
\end{itemize}

Therefore we can
deduce the run-time complexity of the \textsc{NextPriorityConcept‑Strategy}
algorithm: \(O(|{\cal B}|\,|G|\,|P|^2\,c_\sigma)\).

\hypertarget{memory-analysis}{%
\subsubsection{Memory analysis}\label{memory-analysis}}

At each step of the main loop in the \textsc{NextPriorityConcept‑Strategy}
algorithm, a set of predecessors is generated. The cardinality of these predecessors
cannot exceed \(|P|\). These predecessors will not be explored until
all their predecessors have been examined (principle of the priority queue using
the support of concepts). For each concept in the priority queue, a set of
constraints is maintained whose cardinality cannot exceed \(|P|\).
So the memory complexity of the \textsc{NextPriorityConcept‑Strategy}
algorithm is in \(O(w\,|P|^2)\) where \(w\) is the width of the
concept lattice.

\hypertarget{example-predicate}{%
\subsection{Example}\label{example-predicate}}

\begin{longtable}[]{@{}lrrrrr@{}}
\caption{\texttt{Digit} context where
\textbf{c} stands for \textbf{c}omposed,
\textbf{e} for \textbf{e}ven,
\textbf{o} for \textbf{o}dd,
\textbf{p} for \textbf{p}rime and
\textbf{s} for \textbf{s}quare \label{table:digit}}\tabularnewline
\toprule
& \textbf{c} & \textbf{e} & \textbf{o} & \textbf{p} & \textbf{s}\tabularnewline
\midrule
\endfirsthead
\toprule
& \textbf{c} & \textbf{e} & \textbf{o} & \textbf{p} & \textbf{s}\tabularnewline
\midrule
\endhead
0 & \checkmark & \checkmark & & & \checkmark\tabularnewline
1 & & & \checkmark & & \checkmark\tabularnewline
2 & & \checkmark & & \checkmark &\tabularnewline
3 & & & \checkmark & \checkmark &\tabularnewline
4 & \checkmark & \checkmark & & & \checkmark\tabularnewline
5 & & & \checkmark & \checkmark &\tabularnewline
6 & \checkmark & \checkmark & & &\tabularnewline
7 & & & \checkmark & \checkmark &\tabularnewline
8 & \checkmark & \checkmark & & &\tabularnewline
9 & \checkmark & & \checkmark & & \checkmark\tabularnewline
\bottomrule
\end{longtable}

\begin{longtable}[]{@{}rlllr@{}}
\caption{Execution \label{table:digit-execution}}\tabularnewline
\toprule
Step & \(\cal C\) & \(Q\) & \((A,P)\) & \(|A|\)\tabularnewline
\midrule
\endfirsthead
\toprule
Step & \(\cal C\) & \(Q\) & \((A,P)\) & \(|A|\)\tabularnewline
\midrule
\endhead
0 & \{0123456789:\(\emptyset\)\} & {[}(10,\$0){]} & (0123456789,\(\emptyset\)) & 10\tabularnewline
1 & \{04689:eo, 02468:co, 13579:ce\} & {[}(5,\$1), (5,\$2), (5,\$3){]} & (04689,c) & 5\tabularnewline
2 & \{02468:co, 13579:ce, 0468:o\} & {[}(5,\$2), (5,\$3), (4,\$4){]} & (02468,e) & 5\tabularnewline
3 & \{13579:ce, 0468:o\} & {[}(5,\$3), (4,\$4){]} & (13579,o) & 5\tabularnewline
4 & \{0468:o, 357:e, 9:e\} & {[}(4,\$4), (3,\$5), (1,\$7){]} & (0468,ce) & 4\tabularnewline
5 & \{357:e, 9:e, 04:o\} & {[}(3,\$5), (2,\$6), (1,\$7){]} & (357,op) & 3\tabularnewline
6 & \{9:e, 04:o, \(\emptyset\):ceops\} & {[}(2,\$6), (1,\$7), (0,\$8){]} & (04,ces) & 2\tabularnewline
7 & \{9:e, \(\emptyset\):ceops\} & {[}(1,\$7), (0,\$8){]} & (9,cos) & 1\tabularnewline
8 & \{\(\emptyset\):ceops\} & {[}(0,\$8){]} & (\(\emptyset\),ceops) & 0\tabularnewline
9 & \{\} & & {[}{]} &\tabularnewline
\bottomrule
\end{longtable}

\begin{longtable}[]{@{}lrrrrr@{}}
\caption{Context of the lattice with the maximal support strategy
where (\checkmark) stands for the relations that are not considered \label{table:digit-reduced}}\tabularnewline
\toprule
& \textbf{c} & \textbf{e} & \textbf{o} & \textbf{p\textbar o} & \textbf{s\textbar ec}\tabularnewline
\midrule
\endfirsthead
\toprule
& \textbf{c} & \textbf{e} & \textbf{o} & \textbf{p\textbar o} & \textbf{s\textbar ec}\tabularnewline
\midrule
\endhead
0 & \checkmark & \checkmark & & & \checkmark\tabularnewline
1 & & & \checkmark & & (\checkmark)\tabularnewline
2 & & \checkmark & & (\checkmark) &\tabularnewline
3 & & & \checkmark & \checkmark &\tabularnewline
4 & \checkmark & \checkmark & & & \checkmark\tabularnewline
5 & & & \checkmark & \checkmark &\tabularnewline
6 & \checkmark & \checkmark & & &\tabularnewline
7 & & & \checkmark & \checkmark &\tabularnewline
8 & \checkmark & \checkmark & & &\tabularnewline
9 & \checkmark & & \checkmark & & (\checkmark)\tabularnewline
\bottomrule
\end{longtable}

We consider the formal context \texttt{digit} in Table \ref{table:digit} as first example.
The maximal support strategy \(\sigma_{\mbox{max}}\) leads to the lattice
whose Hasse diagram is displayed in Figure
\ref{fig:greatest-support}, where attributes and objects are indicated in respectively the first and the last concept
where they appears. We also indicate the number (using \$) and the support (using \#)
of each concept.
The trace execution is in Table \ref{table:digit-execution}.
The result context \((G,P,I_P)\) displayed in Table \ref{table:digit-reduced}
is clearly a sub-context of the initial one.

We can observe that this second concept lattice contains only \(9\) concepts instead of \(14\).
The concepts for attributes \textbf{p} and \textbf{s} are not generated
as immediate predecessors of the top concept since their support is not maximal.
However, \textbf{p} appears in concept \$5, generated as a predecessor of concept \$3 equal to \((\{1,3,5,7,9\},o)\),
thus introduced only for the \texttt{odd} digits \(\{1,3,5,7,9\}\).
Therefore, \textbf{p} means \texttt{prime} property only for the \texttt{odd} digits, denoted \textbf{p\textbar o}.
In the same way, \textbf{s} appears in concept \$6 as a predecessor of concept \$4 equal to \((\{0,4,6,8\},ec)\),
meaning \texttt{square} property for the \texttt{even} and \texttt{composite} digits, denoted \textbf{p\textbar ec}.

The classical concept lattice produced without strategy is displayed in Figure \ref{fig:digits}.

\begin{figure}
\hypertarget{fig:greatest-support}{%
\centering
\includegraphics[width=0.4\textwidth,height=\textheight]{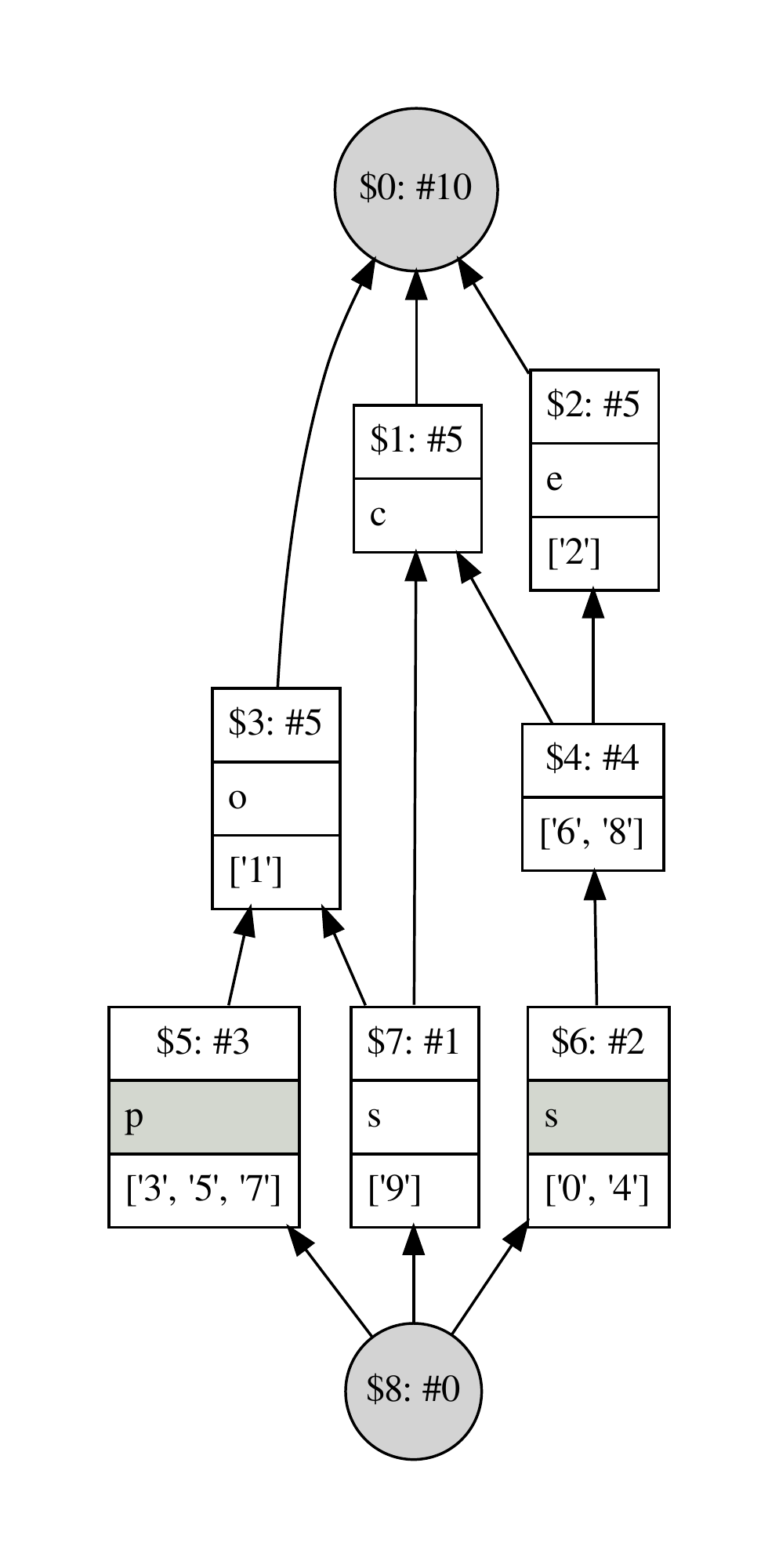}
\caption{\texttt{Digit} sample with greatest support strategy}\label{fig:greatest-support}
}
\end{figure}

\begin{figure}
\hypertarget{fig:digits}{%
\centering
\includegraphics[width=0.6\textwidth,height=\textheight]{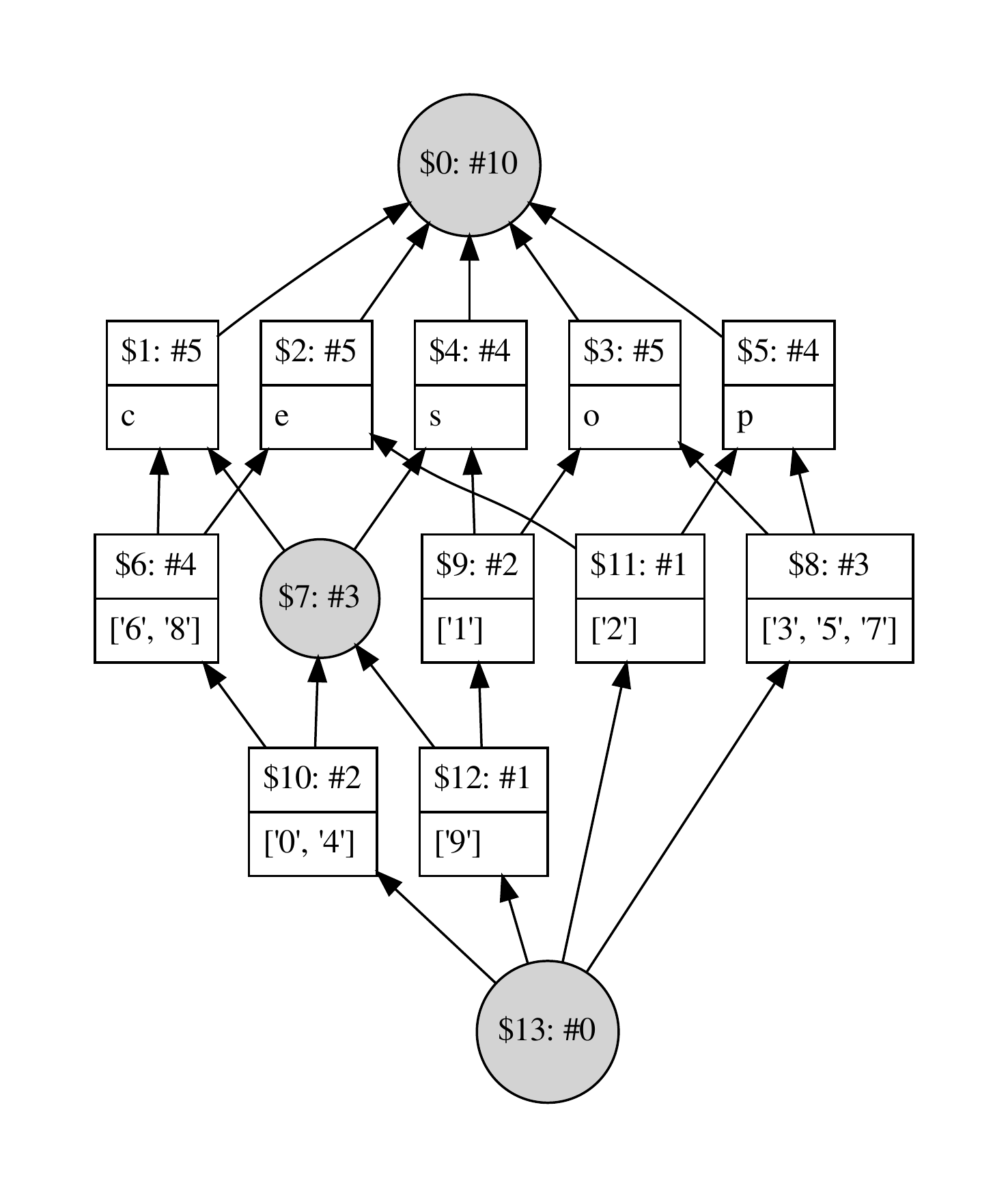}
\caption{\texttt{Digit} sample without strategy}\label{fig:digits}
}
\end{figure}

As second example, we consider the \texttt{Lenses} dataset from the UCI Machine Learning
Repository\footnote{\url{https://archive.ics.uci.edu}}. This dataset is composed of 24
objects/patients described by 4 categorical attributes:

\begin{itemize}
\tightlist
\item
  \emph{age of the patient}: young, pre-presbyopic, presbyopic
\item
  \emph{spectacle prescription}: myope, hypermetrope
\item
  \emph{astigmatic}: no, yes
\item
  \emph{tear production rate}: reduced, normal
\end{itemize}

and classified in 3 classes

\begin{itemize}
\tightlist
\item
  the patient should be fitted with \emph{hard contact lenses},
\item
  the patient should be fitted with \emph{soft contact lenses},
\item
  the patient should \emph{not be fitted} with contact lenses.
\end{itemize}

We consider the formal context composed of \(9\) binary attributes, i.e.~the modalities of the \(4\) categorical attributes.

The classical concept lattice contains 109 concepts.
With the entropy strategy \(\sigma_{\mbox{entropy}}\) using the class information
and by keeping only the two best entropy measures for the predecessors,
we obtain a more compact lattice of 28 concepts displayed in Figure~\ref{fig:lenses}.

As long as a concept contains only one class, no new predecessors are generated
by the strategy.
Therefore these concepts and their ascendants can be interpreted as a clustering
of the data, each concept \((A,D)\) among these clusters corresponding to a class
\(c\), and meaning that objects having attributes \(D\) belong to the class \(c\).

\begin{figure}
\hypertarget{fig:lenses}{%
\centering
\includegraphics{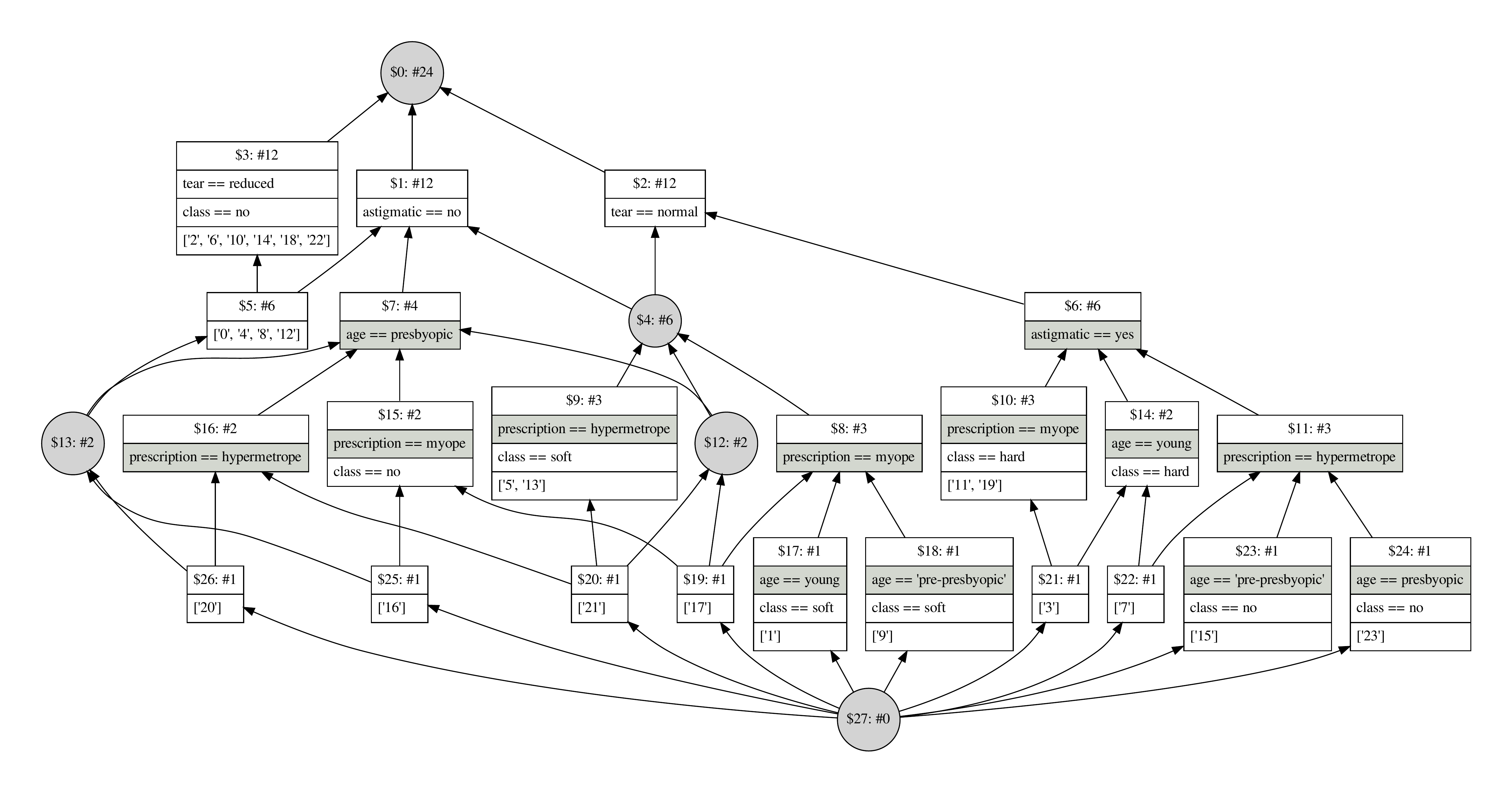}
\caption{\texttt{Lenses} dataset with the entropy strategy}\label{fig:lenses}
}
\end{figure}

\hypertarget{nextpriorityalgorithm-heterogeneous}{%
\section{Next Priority Concept: heterogeneous data as input}\label{nextpriorityalgorithm-heterogeneous}}

\hypertarget{nextpriorityalgorithm}{%
\subsection{Next Priority Concept algorithm: the final version}\label{nextpriorityalgorithm}}

In the previous section, the new set \(P\) of attributes is introduced to
store the selected attributes from which predecessors of a concept \((A,D)\) are
generated. It is easily observed that it is also possible to introduce new
attributes at each iteration without changes. For example, we can consider
\(\sigma_{\mbox{\small neg}}(A)=\{b, \overset{-}{b} \;:\; b \in M-\alpha(A)\}\),
a strategy adding negative attributes as selector.

Our final \textsc{NextPriorityConcept} algorithm exploits this possibility
to manage heterogeneous data as input.
We use predicates describing objects independent of their types.
In order to avoid confusion with classical binary attributes, we will use
characteristic instead of attribute.

More formally, we consider a heterogeneous dataset \((G,S)\) as input where each
\textbf{characteristic} \(s \in S\) can be seen as a mapping \(s: G \rightarrow R_s\) where
\(R_s\) is called the domain of \(s\). Let \(p_s\) be a \textbf{predicate} for a
characteristic \(s\) and \(a\in G\), we note \(p_s(a)\) when \(s(a)\) verifies \(p_s\).

For example, a numerical characteristic can be described by predicates on the
form \emph{is smaller/greater than \(c\)} where \(c\) a numerical value; a
characteristic representing a (temporal) sequence can be described by
predicates of the form \emph{contains \(s\) as (maximal) subsequence} where \(s\)
is a sequence; and a classical boolean characteristic \(b\) corresponds to the
predicate \emph{possesses \(b\)}.

We introduce the notion of \emph{description} \(\delta\) to provide predicates describing
a set of objects, and we extend the notion of strategy \(\sigma\) to provide
predicates (called selectors) to generate (select) the predecessors of a concept.

Characteristics of different domains must be processed separately since
predicates are calculated differently according to the domain.
However, characteristics on the same domain can be processed together or
separately, and some characteristics may not be considered, or considered several times.\\
Therefore, characteristics are given by a family \(S=(S^i)_{i\leq d}\), where each \(S^i\) contains characteristics on the same domain.

For example, for the well-known Iris database\footnote{\url{https://archive.ics.uci.edu}} composed of class information
and four numerical characteristics
\(S=\{\mbox{\texttt{sepal-length}}, \mbox{\texttt{sepal-width}}, \mbox{\texttt{petal-length}}, \mbox{\texttt{petal-width}}\}\),
we can consider the \texttt{petal} characteristics together, and the \texttt{sepal}
characteristics together
(\(S^1=\{\mbox{\texttt{petal-length}}, \mbox{\texttt{petal-width}}\}\) and
\(S^2=\{\mbox{\texttt{sepal-length}}, \mbox{\texttt{setal-width}}\}\)).
We can also only consider the petal characteristics, and separately
(\(S^1=\{\mbox{\texttt{petal-length}}\}\) and \(S^2=\{\mbox{\texttt{petal-width}}\}\)).

Predicates are provided according to a given \(S^i\), both to describe a set of
objects, but also to compute the predecessors of a concept:

\begin{description}
\item[A description \(\delta^i\)]
is an application \(\delta^i:2^G\rightarrow 2^P\) which defines a set
of predicates \(\delta^i(A)\) describing the characteristics of \(S^i\) for any
subset \(A\) of \(G\).
\item[A strategy \(\sigma^i\)]
is an application \(\sigma^i:2^G\rightarrow 2^P\) which defines a set
of predicates \(\sigma^i(A)\) (called selectors) for characteristics of \(S^i\) from which the
predecessors of a concept \((A,D)\) are generated.
\end{description}

Therefore, the strategy \(\sigma(A)\) and the description \(\delta(A)\) of a subset
\(A\) of objects are defined from \(2^G\) to \(2^P\) by:
\[\delta(A) = \bigcup_{S^i}\delta^i(A)\]
\[\sigma(A) = \bigcup_{S^i}\sigma^i(A)\]

Let us give some examples of strategies and descriptions for a set \(A\subseteq G\) of objects:

\begin{itemize}
\tightlist
\item
  For a numerical attribute \(S^i=\{s\}\):

  \begin{itemize}
  \tightlist
  \item
    \(\delta^i(A)=\{ \mbox{\emph{is greater than }} \min_{a\in A}s(a), \mbox{\emph{is smaller than }} \max_{a\in A}s(a)\}\)
  \item
    \(\sigma^i(A)=\{ \mbox{\emph{is greater than }} q_1, \mbox{\emph{is smaller than }} q_3\}\)
    where \(q_1\) and \(q_3\) are respectively the first and the third quantile of the values \((s(a))_{a\in A}\).
  \end{itemize}
\item
  For a sequential attribute \(S^i=\{s\}\):

  \begin{itemize}
  \tightlist
  \item
    \(\delta^i(A)=\{\mbox{\emph{ contains }} X \mbox{\emph{ as subsequences}}\}\) where \(X\) is the set
    of longest most common subsequences of \(\{s(a)\}_{a\in A}\).
  \item
    \(\sigma^i(A)=\{\mbox{\emph{ contains }} X' \mbox{\emph{as subsequences}}\}\) where sequences of
    \(X'\) are obtained by a reduction process of the longest most common
    subsequences of \(\delta^i(A)\).
  \end{itemize}
\end{itemize}

\begin{algorithm}[h]
\label{next-priority-concept}
\SetKw{KwNot}{not}
\SetKw{KwEmpty}{empty}
\SetKw{KwDelete}{delete}
\SetKw{KwProduce}{produce}
\KwData{
\begin{itemize}
\item $\langle G,S\rangle$ a dataset
\item $(S^i)_{i\leq d}$ a family of $S$
\item $\delta$ a description
\item $\sigma$ a strategy
\end{itemize}
}
\KwOut{
\begin{itemize}
\item the formal context $\langle G,P,I_P\rangle$
\item its concepts $(A,D)$
\end{itemize}
}
\Begin{
    \texttt{/* Priority queue for the concepts */}\\ 
    $Q \leftarrow \mbox{\texttt{[]}}$ \tcc*{$Q$ is a priority queue using the support of concepts}
    $Q$.push($(|G|,(G,\delta(G))$) \tcc*{Add the top concept into $Q$}
    \texttt{}\\ 

    \texttt{/* Data structure for constraints */}\\ 
    ${\cal C} \leftarrow \mbox{\texttt{[]}} $ \tcc*{${\cal C}$ is the descendant constraints map being $\emptyset$ by default}

    \texttt{/* Data structures for predicates */}\\ 
    $P \leftarrow \emptyset$    \tcc*{$P$ is the set of all predicates} 
    $I_P \leftarrow \emptyset$ \tcc*{$I_P$ is the binary relation between $G$ and $P$}

    \texttt{/* Immediate predecessors generation */}\\ 
    \While{$Q$ \KwNot \KwEmpty}{
        $(A,D) \leftarrow Q.\mbox{pop()}$ \tcc*{Get the concept with highest support}
        \KwProduce $(A,D)$ \;
        $LP \leftarrow$ \textsc{Predecessors‑Desc}($(A,D)$, $P$, $I_P$, ${\cal C}$, $\sigma$, $\delta$) \;
        
        \texttt{/* Update queue */}\\
        \ForAll{$(A',D') \in LP$}{
            \If{$(A',D')\not\in Q$}{
            $Q$.push($(|A'|,(A',D'))$) \tcc*{Add concept into $Q$}
        }}
        \KwDelete ${\cal C}[A]$ \tcc*{Remove useless data}
    }
    \Return {$\langle G,P,I_P\rangle$}
    }
\caption{\textsc{NextPriorityConcept}}
\end{algorithm}

\begin{center}\rule{0.5\linewidth}{\linethickness}\end{center}

\textsc{NextPriorityConcept} and \textsc{Predecessors‑Description}
algorithms are very similar to the previous versions.
\textsc{NextPriorityConcept} algorithm considers as input

\begin{itemize}
\tightlist
\item
  a heterogeneous dataset \((G,S)\)
\item
  \((S^i)_{i\leq d}\) a family of \(S\)
\item
  \(\delta\) a description
\item
  \(\sigma\) a strategy
\end{itemize}

This algorithm computes the formal
context \(\langle G,P,I_P\rangle\) and its concepts, where \(P\) is
the set of predicates describing the characteristics, \(I_P=\{(a,p) \;:\; p(a)\}\)
is the relation between objects and predicates, and \((\alpha_P,\beta_P)\) is
the associated Galois connection.

\textsc{Predecessors‑Description} is a modified version of
\textsc{Predecessors‑Strategy} to compute predecessors of a concept \((A,D)\),
where we consider \(\sigma(A)\) and \(\delta(A)\) given as input.

\begin{algorithm}[h]
\label{predecessors-description}
\SetKw{KwTrue}{true}
\SetKw{KwFalse}{false}
\SetKw{KwBreak}{break}
\KwData{
\begin{itemize}
\item $(A,D)$ a concept
\item $P$ the set of predicates
\item $I_P$ the binary relation between $G$ and $P$
\item ${\cal C}$ the constraints
\item $\sigma$ a strategy (issued from the $\sigma^i$)
\item $\delta$ a description (issued from the $\delta^i$)
\end{itemize}
}
\KwResult{
\begin{itemize}
\item $LP$ a set of predecessors
\end{itemize}
}
\Begin{
    $L \leftarrow \emptyset$\;
    
    \ForAll{$p \in (\sigma(A) \cup {\cal C}[A]) \setminus D$}{
                        \texttt{/* p is a new ``potential'' selector to generate a predecessor */}\\ 
                $A'\leftarrow \{a\in A\;:\; p(a)\}$ \tcc*{$A'$ are the objects verifying $D+p$}
                \texttt{/* Add $(A',p)$ if $A'$ maximum in $L$ and included in $A$ */}\\
                \lIf{$A'\subset A$}{$L \leftarrow$ \textsc{Inclusion‑Max}($L$,$(A',p)$)}
        }
        
    $N\leftarrow \{p\;:\; (A',p)\in L\}$ \tcc*{$N$ is the set of new constraints}

$LP \leftarrow \emptyset$\;
        
\ForAll{$(A',p')\in L$}{
        \texttt{/* Update the selected attributes $P$ and the relation $I_P$*/}\\
        \If{$p'\in \sigma(A)$} {
        $P \leftarrow P \cup \{p'\}$ \tcc*{Update the set of selected predicates}
        }
        $D'\leftarrow \delta(A')$ \tcc*{$D'$ are the new predicates describing $A'$}
        $LP$.add($(A',D')$) \tcc*{$(A',D')$ is a new concept}
        $I_P \leftarrow I_P \cup (A'\times D')$ \tcc*{Update the new relation}

        \texttt{/* Compute cross constraints ($X$) and propagate constraints */}\\
        $X\leftarrow \{p''\in N\;:\; p''(a)\;\forall a \in A'\}$\\
        ${\cal C}[A'] \leftarrow {\cal C}[A'] \cup {\cal C}[A] \cup N \setminus X$\\
        
        }

    \Return{$LP$}
}
\caption{\textsc{Predecessors‑Desc}}
\end{algorithm}

\hypertarget{discussions}{%
\subsection{Discussions}\label{discussions}}

\hypertarget{comparison-with-pattern-structures}{%
\subsubsection{Comparison with pattern structures}\label{comparison-with-pattern-structures}}

The predicates in the final set \(P\) are those issued from the
descriptions since the predicates generated by the strategy are only used to
generate predecessors. Our \textsc{NextPriorityConcept} algorithm can be
interpreted as a pattern structure on each domain of characteristics \(S^i\).

Formally, a pattern structure \citep{Kuznetsov_2001} is a triple
\((G,({\cal D},\sqcap),\delta)\) where \(G\) is a set of objects,
\(({\cal D},\sqcap)\) is a meet semi-lattice of potential objects descriptions,
and \(\delta: G \rightarrow {\cal D}\) associates to each object its
description. Elements of \({\cal D}\) are ordered with the subsumption
relation \(\sqsubseteq\).

Let \((2^G,{\subseteq})\galois{\alpha_{\cal D}}{\beta_{\cal D}}({\cal D},\sqcap)\)
be the corresponding Galois connection where:

\begin{itemize}
\tightlist
\item
  \(\alpha_{\cal D}: 2^G\rightarrow {\cal D}\) is defined for
  \(A\subseteq G\) by \(\alpha_{\cal D}(A)= \sqcap_{g\in A} \delta(g)\).
\item
  \(\beta_{\cal D}: {\cal D}\rightarrow 2^G\) is defined, for
  \(d\in {\cal D}\) by
  \(\beta_{\cal D}(d)= \{g\in G\;:\; d\sqsubseteq \delta_{\cal D}(d)\}\).
\end{itemize}

Pattern concepts are pairs \((A,d)\), \(A\subseteq G\), \(d\in {\cal D}\) such that
\(\alpha_{\cal D}(A)=d\) and \(A=\beta_{\cal D}(d)\). \(d\) is a pattern intent, and
is the common description of all objects in \(A\). When partially ordered by
\((A_1,d_1)\leq (A_2,d_2) \iff A_1\subseteq A_2 (\iff d_2 \sqsubseteq d_1)\),
the set of all pattern concepts forms a lattice called the \emph{pattern lattice}.

In our case, we have an implicit description space \({\cal D}^i\) by \(\delta^i\)
for each subset \(S^i\) of characteristics, where the description \(\delta^i(A)\)
of a set \(A\subseteq G\) is a direct translation by predicates of its description
in \({\cal D}^i\). A nice result in pattern structure establishes that there is a
Galois connection between \(G\) and \({\cal D}^i\)
if and only if \(({\cal D}^i,\sqcap)\) is a meet semi-lattice.

In order to maintain the final Galois connection \((\alpha_P,\beta_P)\)
between objects and predicates, each description \(\delta^i(A)\) must verify
\(\delta^i(A)\sqsubseteq\delta^i(A')\) for \(A'\subseteq A\) since
\(\delta^i(A)\sqsubseteq\delta^i(A') \iff \delta^i(A)\sqcap\delta^i(A') = A\).
Therefore, we can state the following results as a corollary of
Theorem~\ref{th:algo}:

\hypertarget{th:algo}{}
\begin{corollary} \label{th:algo}

If each description \(\delta^i\) verifies \(\delta^i(A)\sqsubseteq\delta^i(A')\) for
\(A'\subseteq A\), then \textsc{NextPriorityConcept} algorithm computes the concept
lattice of \(\langle G,P,(\alpha_P,\beta_P)\rangle\) with a run-time in
\(O(|{\cal B}|\,|G|\,|P|^2\,(c_{\sigma}+c_{\delta}))\) (where \({\cal B}\)
is the number of concepts, \(c_{\sigma}\) is the cost of the strategy and \(c_{\delta}\) is
the cost of the description,
and a space memory in \(O(w\,|P|^2)\) (where
\(w\) is the width of the concept lattice).

\end{corollary}

\hypertarget{th:algo}{}
\begin{corollary} \label{th:algo}

If each description \(\delta^i\) verifies \(\delta^i(A)\sqsubseteq\delta^i(A')\) for
\(A'\subseteq A\), then

\begin{center}
$(2^G,{\subseteq})\galois{\delta}{\pi}(2^P,\sqsubseteq)$ is a Galois connection 
\end{center}

where \(\pi: 2^P\rightarrow 2^G\) is defined, for
\(D\subseteq P\) by
\(\pi(D)= \{a\in G\;:\; p(a) \;\forall p\in D\}\).
And \(\pi\circ\delta\) is a closure operator on \(P\).

\end{corollary}

While patterns are globally computed in a preprocessing step using pattern
structures, our \textsc{NextPriorityConcept} algorithm is a pattern discovery approach
where predicates are discovered ``on the fly'', in a local way for each concept.
This is made possible
by the use of the priority queue (to ensure that each concept is generated
before its predecessors) and the propagation of constraints (to ensure that meet
will be computed). Therefore, predicates are well-suited to the data, and
lattices are often smaller, with more relevant concepts.
Moreover, the use of predicates mixed with specialized strategies and
descriptions on each domain of characteristics allows mining of
complex and heterogeneous data.

\hypertarget{processing-of-group-of-characteristics}{%
\subsubsection{Processing of group of characteristics}\label{processing-of-group-of-characteristics}}

When some characteristics are defined on the same domain, the family
\((S^i)_{i\leq d}\) offers the possibility to process them separately or together.
An immediate way to process with several characteristics together would be to
merge the predicates obtained in the individual case, both for the descriptions
and for the strategies. But it is possible to obtain more relevant predicates
by a specific process of a group of characteristics.

For example, for a group of \(k\) numerical characteristics \(s_1,\ldots s_k\), we
can consider the \(k\)-dimensional points \(\{ (s_j(a))_{j\leq k}\;:\; a\in A\}\)
for a set \(A\) of objects, and their convex hull \citep{belfodil:hal-01573841}. The description \(\delta^i(A)\)
is then composed of predicates describing the borders of the convex hull, and
the strategy \(\sigma^i(A)\) is a way to cut the hull. For points in two
dimensions, the convex hull is a polygon, and borders and cuts are lines.
Clearly, for two sets \(A\) and \(A'\) of objects such that \(A'\subseteq A\), the
convex hull of \(A'\) is included into the convex hull of \(A\), and the
intersection of two convex hulls is a convex hull.
Therefore \(\delta^i(A)\sqsubseteq\delta^i(A')\)

For points in two and three dimensions, output-sensitive algorithms are known
to compute the convex hull in time \(O(n \log n)\), where \(n\) is the number of points.
For dimensions \(d\) higher
than \(3\), the time for computing the convex hull is
\(O(n^{\lfloor d/2\rfloor })\) \citep{ConvexHull_1986}. This process therefore
impacts on the costs \(c_{\sigma}\) and \(c_{\delta}\).

Now consider a group of \(k\) boolean characteristics \(S^i=\{x_1,\ldots,x_k\}\).
The classical FCA approach describes a set of objects \(A\) by the set of
attributes \(B=\{x_j\;:\; a\in A \mbox{ and }x_j(a)=1\}\) and the strategy of
generation of immediate predecessors considers the set of all other attributes
\(\{x\in S^i\setminus B\}\) as selectors.
These two sets described by predicates of the form
\emph{possesses attribute} \(x\) would respectively corresponds to
\(\delta^i(A)\) and \(\sigma^i(A)\).

The use of predicates, and especially the possibility of introducing negative attributes,
allows us to consider other descriptions of \(A\).
For example, we can consider a description \(\delta^i(A)\) by predicates for the disjunction of clauses:

\[
\bigvee_{a\in A}\bigwedge_{j\leq k}
\left\{
  \begin{array}{rr}
    x_j & \mbox{ if } x_j(a)=1 \\
    \overline{x_j} & \mbox{ if } x_j(a)=0 \\
  \end{array}
\right.
\]

For a finer and minimal description, we can also consider the minimization of
this boolean formulae using the well-known Quine-McCluskey algorithm (or the
method of prime implicants), with a time complexity in \(O(3^n\log n)\)
\citep{Quine_1952} where \(n\) is the number of attributes.

\hypertarget{about-strategies}{%
\subsubsection{About strategies}\label{about-strategies}}

A strategy proposes a way to \emph{cut} the description \(\delta^i(A)\) by selectors
from which predecessors of a concept \((A,P)\) are generated.
These selected predicates are only used in this way at each step of the algorithm,
but are not kept in the final set \(P\) of predicates, and several strategies are possible to generate predecessors of a concept \((A,P)\).

Therefore, our algorithm can be extended to improve the strategy management:

\begin{description}
\item[Meta-strategy:]
The strategy \(\sigma\) is defined as the union of the strategies
\((\sigma^i)_{i\leq d}\) for each part \(S^i\) of attributes. It is possible to
introduce a filter (or meta-strategy) on these selectors, as those
introduced in the section \ref{nextpriorityalgorithm-strategy}:

\begin{itemize}
\item
  The maximal support meta-strategy relies on the support:

  \[\sigma_{\mbox{\small max}}(A)=\{p\in \bigcup_{S^i}\delta^i(A)\;:\;
    |\pi(p)|\mbox{ maximal}\}\]
\item
  The entropy meta-strategy is a supervised strategy where objects have a
  class attribute:

  \[\sigma_{\mbox{\small entropy}}(A)=\{p\in \bigcup_{S^i}\delta^i(A)\;:\;
    H_{\mbox{class}}(\pi(p))\mbox{ minimal}\}\]
\end{itemize}
\item[Interactivity:]
Several strategies are possible to generate predecessors of a concept, going from the naive strategy
\(\sigma^i_{\mbox{naive}}\) that generates all the possible predecessors, to the
silly strategy \(\sigma^i_{\mbox{silly}}=\emptyset\) that generates no predecessors.
Therefore we can extend our algorithm in an
interactive way, where the user could choose or test several strategies
for each concept in an user driven pattern discovery approach.

In classical FCA approach, the naive strategy considers all the possible attributes
of \(M\setminus B\) for a concept \((A,B)\), and the corresponding lattice is often too large.
The silly strategy allows to introduce some attributes in concepts, but without
considering them in the predecessor generation. This approach is interesting for
example for class attributes.
Every possible strategy is between \(\sigma^i_{\mbox{naive}}\) and \(\sigma^i_{\mbox{silly}}\)
when considering the set of generated predecessors,
and it would be interesting to investigate the whole set of possible strategies.\\
A strategy close to \(\sigma^i_{\mbox{naive}}\) increases the number of concepts,
while a strategy close to \(\sigma^i_{\mbox{silly}}\) decreases the number of concepts.
\end{description}

\hypertarget{example-heterogeneous}{%
\subsubsection{Examples}\label{example-heterogeneous}}

\hypertarget{iris-dataset-with-heterogeneous-descriptions}{%
\paragraph{\texorpdfstring{\texttt{Iris} dataset with heterogeneous descriptions}{Iris dataset with heterogeneous descriptions}}\label{iris-dataset-with-heterogeneous-descriptions}}

\begin{figure}
\hypertarget{fig:iris-heterogenous}{%
\centering
\includegraphics{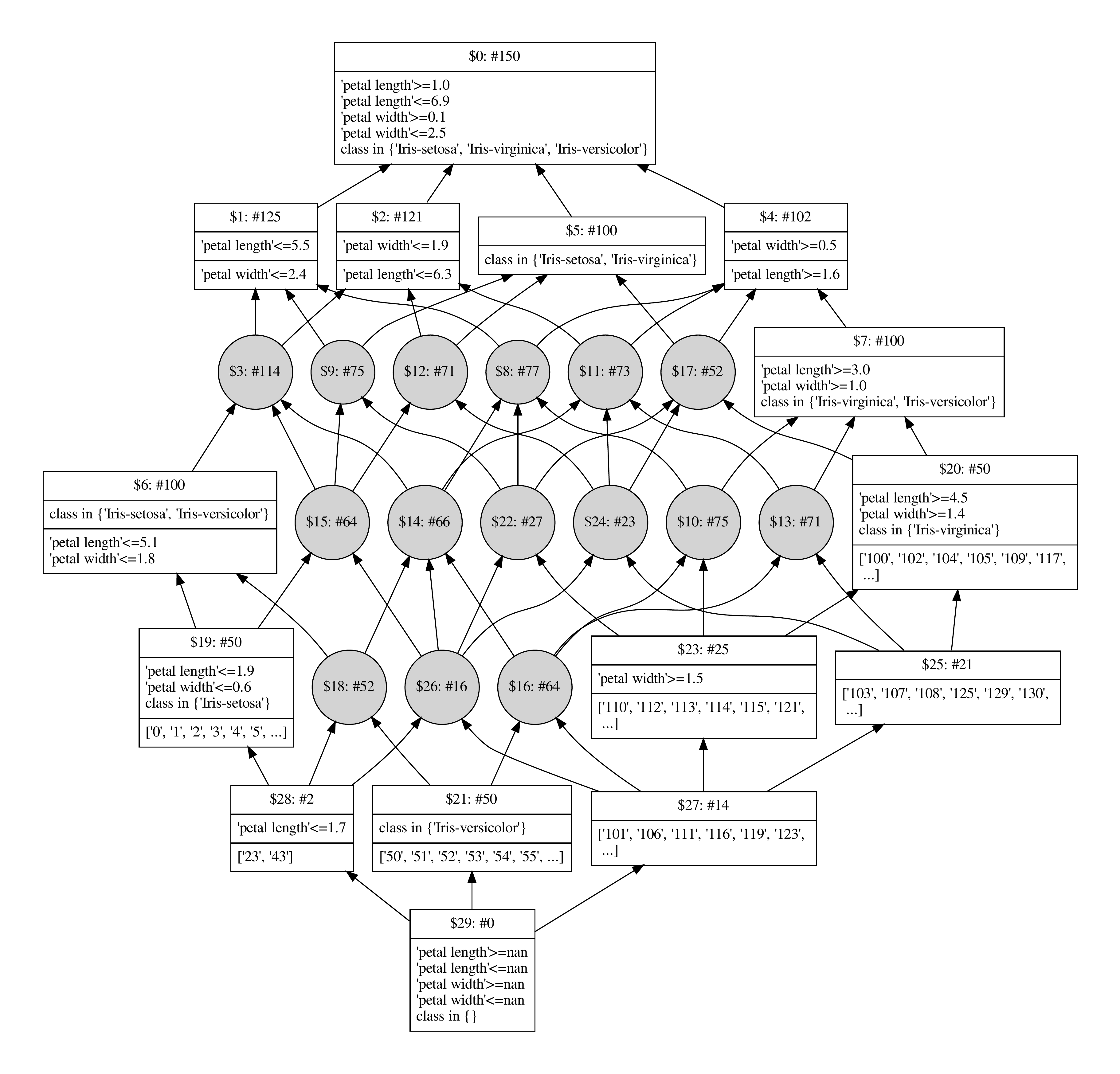}
\caption{\texttt{Iris} dataset with a minimum support (100) strategy using \texttt{petal\ lengh}, \texttt{petal\ width} and \texttt{class}}\label{fig:iris-heterogenous}
}
\end{figure}

Consider the well-known \texttt{Iris} dataset from the UCI Machine Learning
Repository\footnote{\url{https://archive.ics.uci.edu}}, and composed of 150 objects
described by 4 numerical characteristics \texttt{sepal-length}, \texttt{sepal-width},
\texttt{petal-length}, \texttt{petal-width} and classified in 3 classes
\texttt{Setosa}, \texttt{Versicolor}, \texttt{Virginica}.

In this first example, we consider the two \texttt{petal} characteristics separately,
and the \texttt{class} characteristic, thus a combination of two numerical characteristic
with a categorical one.
For each \texttt{petal} characteristics, we use a classical description
by the two predicates \emph{the values are greater than the min} and
\emph{the values are smaller than the max}.
For the \texttt{class} charcateristic, we use the predicate \emph{belongs to class}.

The strategy generates the two selectors
\emph{is the value greater than the mean minus the standard deviation?} and
\emph{is the value smaller than the mean plus the standard deviation?},
and is combined with a meta-strategy limiting new predecessors to those whose support is greater than 100.
combined to the \texttt{class} characteristic,

We obtain the concept lattice displayed in Figure~\ref{fig:iris-heterogenous}
composed of \(30\) concepts:
* concept \$19 corresponds to objects whose class is \texttt{Setosa}
* concept \$20 corresponds to objects whose class is \texttt{Virginica}
* concept \$21 corresponds to objects whose class is \texttt{Versicolor}

\hypertarget{iris-dataset-with-the-entropy-strategy}{%
\paragraph{\texorpdfstring{\texttt{Iris} dataset with the entropy strategy}{Iris dataset with the entropy strategy}}\label{iris-dataset-with-the-entropy-strategy}}

\begin{figure}
\hypertarget{fig:iris-entropy}{%
\centering
\includegraphics{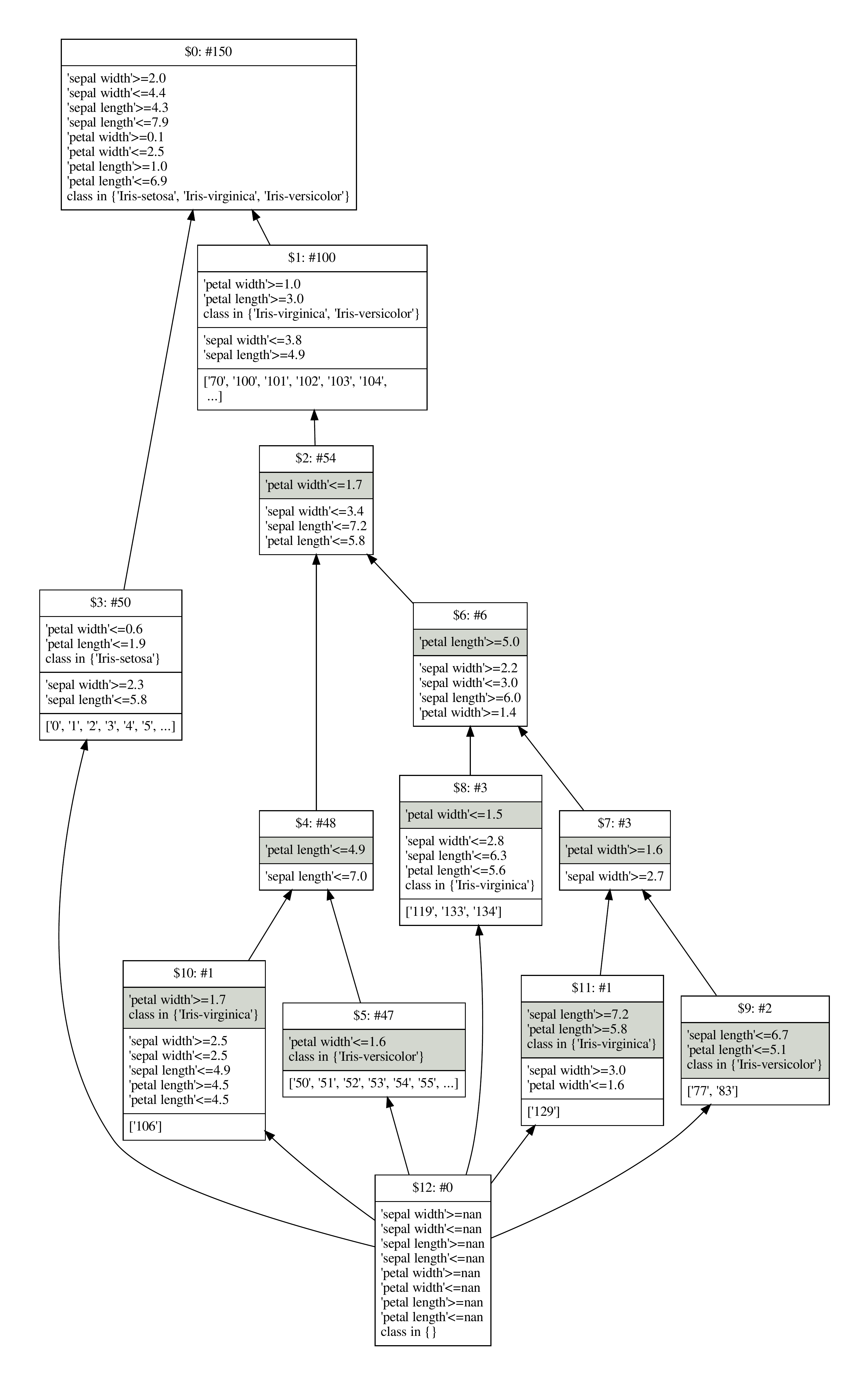}
\caption{\texttt{Iris} dataset with the entropy strategy}\label{fig:iris-entropy}
}
\end{figure}

In the second example, we consider the \texttt{Iris} dataset and the four \texttt{petal}
and \texttt{sepal} characteristics seprately.
We use the following entropy strategy which allows to consider the entropy of
a predecessor \(A'\) of \(A\), but also the entropy of the remaining set \(A\setminus A'\) :

\[H = \theta(H_{A'}) + (1-\theta)H_{A-A'}\]

The more the value of \(\theta\) increases, the more the number of predessors of \(A\) decreases.

\begin{figure}
\hypertarget{fig:iris-scatterplot}{%
\centering
\includegraphics[width=0.8\textwidth,height=\textheight]{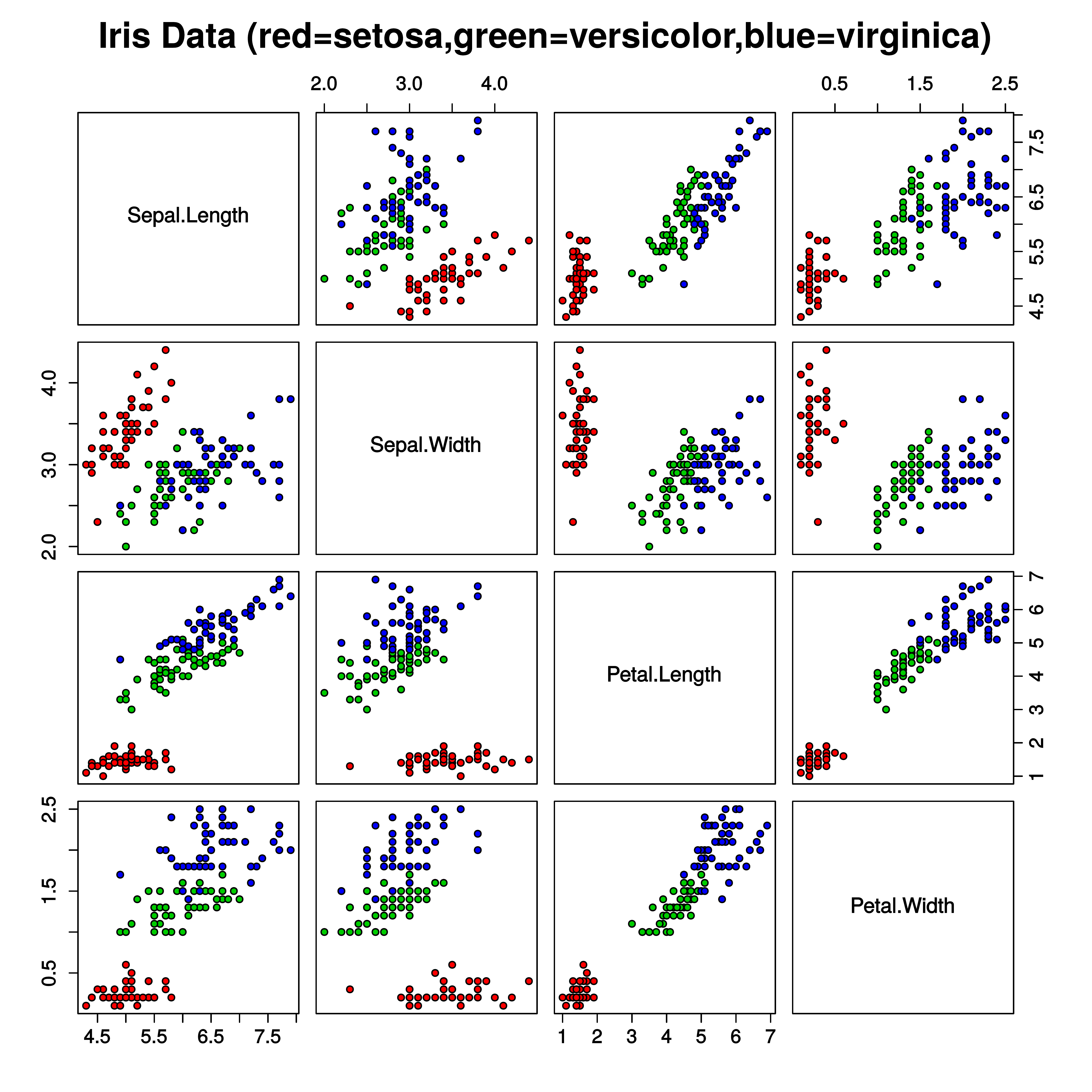}
\caption[\texttt{Iris} dataset]{\texttt{Iris} dataset\footnotemark{}}\label{fig:iris-scatterplot}
}
\end{figure}
\footnotetext{\url{https://commons.wikimedia.org/wiki/File:Iris_dataset_scatterplot.svg}}

We obtain the concept lattice displayed in Figure~\ref{fig:iris-entropy}
(with \(\theta=1/2\)) composed of 13 concepts.
The \texttt{Setosa} iris are quickly separated according to their two \texttt{petal} characteristics (concept \$3).
Indeed, we can observe on the scatterplot in Figure \ref{fig:iris-scatterplot} that this class is
clearly separated from the two others.
We obtain \(4\) concepts for classes \texttt{virginica} and \texttt{versicolor}:

\begin{itemize}
\item
  concept \$10 and \$11 correspond to objects whose class is \texttt{Virginica},
  with only the two \texttt{petal} characteristics used in concept \$10,
  while the \texttt{sepal-length} characteristic is introduced in concept \$11.
\item
  concept \$5 and \$9 correspond to objects whose class is \texttt{Versicolor},
\end{itemize}

\hypertarget{numbers-dataset-with-gcd-and-lcm-as-descriptions}{%
\paragraph{\texorpdfstring{\texttt{Numbers} dataset with GCD and LCM as descriptions}{Numbers dataset with GCD and LCM as descriptions}}\label{numbers-dataset-with-gcd-and-lcm-as-descriptions}}

In this example, we consider the numbers \([36, 48, 56, 64, 84]\) and the
\emph{greatest common divisor (GCD)} and \emph{least common multiple (LCM)} as descriptions.
The strategy consists in adding a new \emph{is divisor of} or \emph{is multiple of} predicate
using a combination of the prime numbers of the set of objects present in the
concept.

\begin{figure}
\hypertarget{fig:numbers-gcd-lcm}{%
\centering
\includegraphics{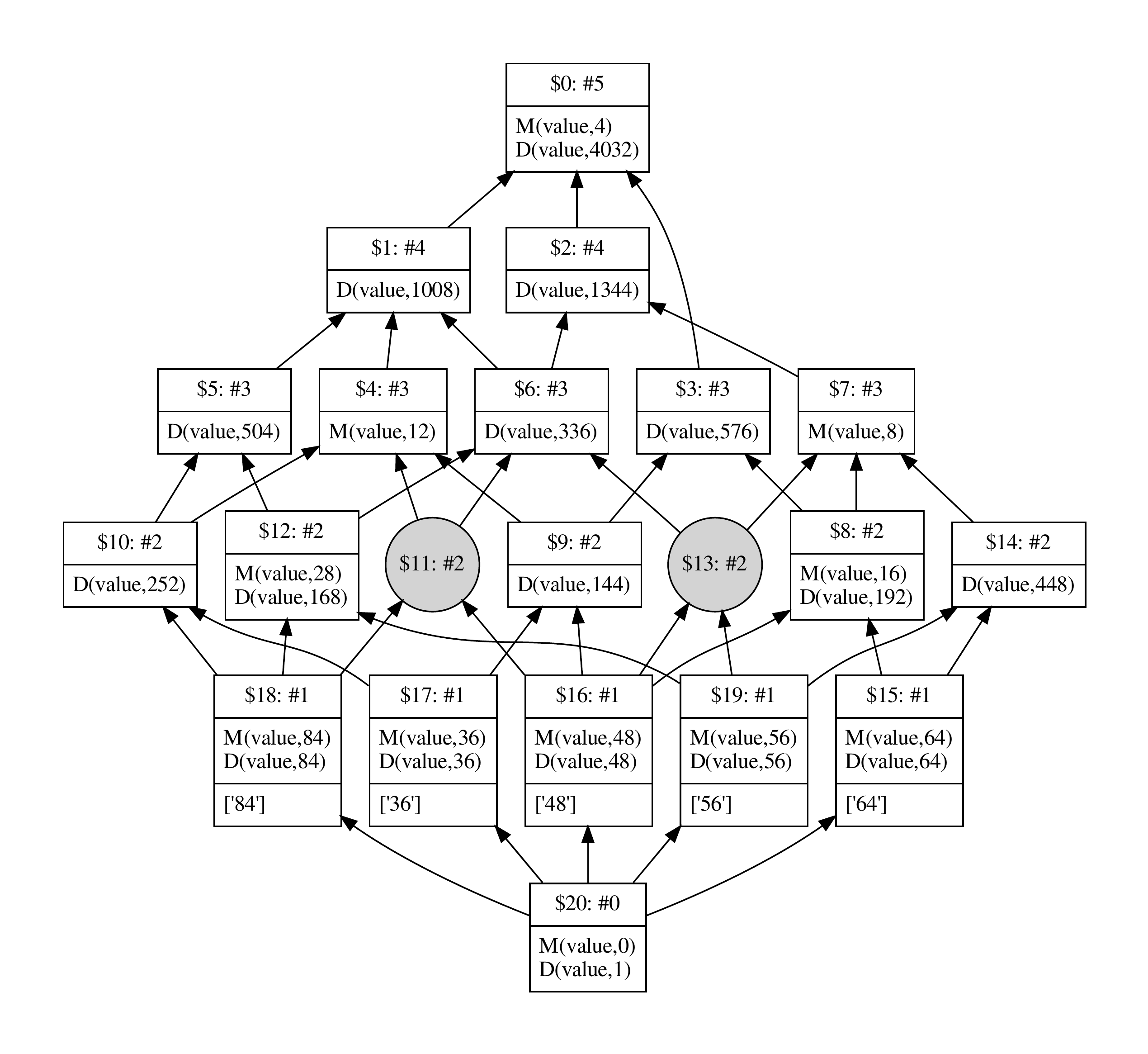}
\caption{\texttt{Numbers} dataset with GCD and LCM as descriptions (M(value, \(x\)) represents the fact that the value is a multiple of \(x\) and D(value, \(x\)) represents the fact that the value is a divisor of \(x\)}\label{fig:numbers-gcd-lcm}
}
\end{figure}

The concept lattice, displayed in Figure~\ref{fig:numbers-gcd-lcm}, is composed of 21 concepts.
The last concept which is the absurd one (the extent is an empty set)
is described by 2 predicates whose conjunction is always false.

\hypertarget{digit-dataset-with-minimal-logical-formula-as-descriptions}{%
\paragraph{\texorpdfstring{\texttt{Digit} dataset with minimal logical formula as descriptions}{Digit dataset with minimal logical formula as descriptions}}\label{digit-dataset-with-minimal-logical-formula-as-descriptions}}

As last example, we consider the \texttt{digit} described by their properties
\textbf{c}omposed, \textbf{e}ven, \textbf{o}dd, \textbf{p}rime and \textbf{s}quare given in Table \ref{table:digit}.

These \(5\) characteristics are considered together, and we compute their minimal boolean formulae
as description using the Quine-McCluskey algorithm \citep{Quine_1952}.
The strategy consists in trying to add an attribute or its negation at each step.

The concept lattice is displayed in Figure~\ref{fig:digit-boolean-hull}.
We can observe that the maximum number of predecessors cannot exceed 5 since the
predecessors are maximum per inclusion.

\begin{figure}
\hypertarget{fig:digit-boolean-hull}{%
\centering
\includegraphics{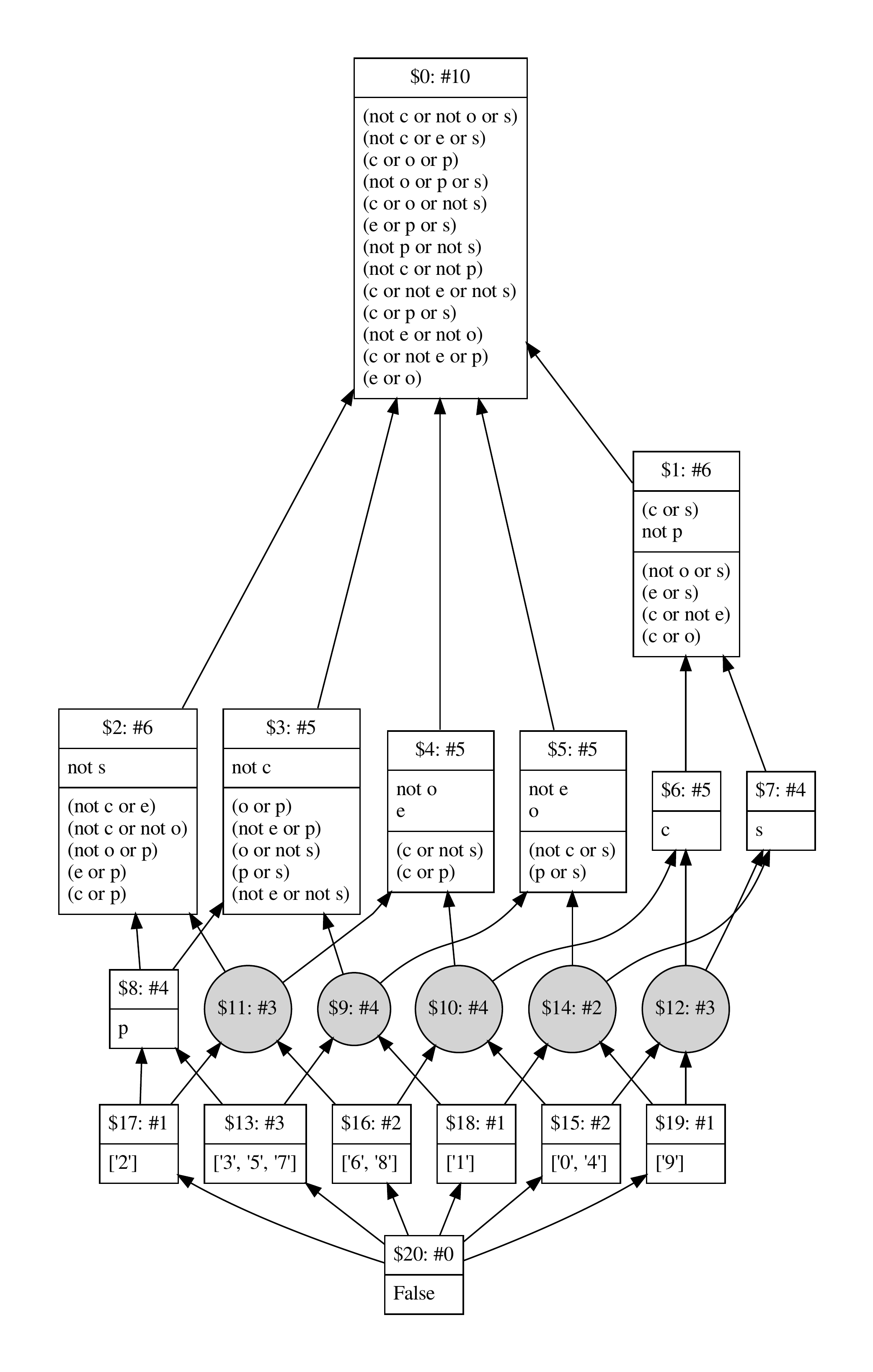}
\caption{\texttt{Digit} dataset with minimal logical formula as description}\label{fig:digit-boolean-hull}
}
\end{figure}

\hypertarget{conclusion}{%
\section{Conclusion}\label{conclusion}}

We have described our \textsc{NextPriorityConcept} algorithm for complex and heterogenous mining
using a pattern discovery approach.

More precisely, our algorithm generates formal concept using the dual version
of Bordat's theorem for the generation
of immediate predecessors (instead of immediate successors), and
where recursion is replaced by a
priority queue using the support of concepts to make sure that concepts are
generated level by level, each concept being generated before its predecessors.
Moreover, a constraint propagation mechanism ensures that meets are correctly
generated.

Heterogeneous data are provided at input with a description mechanism and a
predecessor generation strategy adapted to each kind of data, and generically
described by predicates.

Our algorithm is generic and agnostic since we use predicates whatever the characteristics.
It is implemented with a system of plugins
for an easy integration of new characteristics,
new description, new strategies and new meta-strategies.
We are currently
working on a code diffusion via a development platform, called Galactic
(\textbf{GA}lois \textbf{LA}ttices, \textbf{C}oncept \textbf{T}heory, \textbf{I}mplicational systems and
\textbf{C}losures)\footnote{\url{https://galactic.univ-lr.fr}}.

We have already implemented some descriptions and strategies plugins for
boolean, numeric, categorical attributes, strings and sequences.
We are currently working on descriptions
and strategies for graphs and triadic data.

Acknowledgements: Thanks are owed to Rokia Missaoui and Gaël Lejeune for their constructive and helpfull comments.

\renewcommand\refname{References}
\bibliography{refs.bib}

\end{document}